\documentclass[10pt, a4paper, oneside]{amsart}

\usepackage[left=3cm, right=3cm, ]{geometry}

\usepackage{amsmath}
\usepackage[dvipdfmx]{graphicx}
\usepackage{color}
\usepackage[colorlinks=true, allcolors=blue]{hyperref}
\usepackage{amssymb}
\usepackage{amsthm}
\usepackage{braket}
\usepackage[all]{xy}
\usepackage{mathrsfs}
\usepackage{bm}
\usepackage{ulem}
\usepackage{comment}
\usepackage{here}

\title{Convergence of neural networks to Gaussian mixture distribution }

\author[Y. Asao]{Yasuhiko Asao\textsuperscript{1}}
\address[Asao]{Department of Applied Mathematics\\ Fukuoka University \\ 8-19-1\\Nanakuma\\Jonan-ku\\Fukuoka city\\ Japan}
\email{asao@fukuoka-u.ac.jp}

\author[R. Sakamoto]{Ryotaro Sakamoto\textsuperscript{2}}
\address[Sakamoto]{Department of Mathematics\\University of Tsukuba\\1-1-1 Tennodai\\Tsukuba\\Ibaraki 305-8571 Japan}
\email{rsakamoto@math.tsukuba.ac.jp}

\author[S. Takagi]{Shiro Takagi\textsuperscript{3}}
\address[Takagi]{Independent Researcher}
\email{takagi4646@gmail.com}

\begin{document}
\maketitle
\theoremstyle{definition}
\newtheorem{theorem}{Theorem}[section]
\newtheorem{definition}[theorem]{Definition}
\newtheorem{proposition}[theorem]{Proposition}
\newtheorem{lemma}[theorem]{Lemma}
 \newtheorem{corollary}[theorem]{Corollary}
 \newtheorem{remark}[theorem]{Remark}
 \newtheorem{example}[theorem]{Example}
 \newtheorem{assumption}[theorem]{Assumption}

\newcommand{\Hom}{{\rm Hom}}
\newcommand{\Gal}{{\rm Gal}}
\newcommand{\SL}{{\rm SL}}

\newcommand{\bA}{\mathbb{A}}
\newcommand{\bB}{\mathbb{B}}
\newcommand{\bC}{\mathbb{C}}
\newcommand{\bD}{\mathbb{D}}
\newcommand{\bE}{\mathbb{E}}
\newcommand{\bF}{\mathbb{F}}
\newcommand{\bG}{\mathbb{G}}
\newcommand{\bH}{\mathbb{H}}
\newcommand{\bI}{\mathbb{I}}
\newcommand{\bJ}{\mathbb{J}}
\newcommand{\bK}{\mathbb{K}}
\newcommand{\bL}{\mathbb{L}}
\newcommand{\bM}{\mathbb{M}}
\newcommand{\bN}{\mathbb{N}}
\newcommand{\bO}{\mathbb{O}}
\newcommand{\bP}{\mathbb{P}}
\newcommand{\bQ}{\mathbb{Q}}
\newcommand{\bR}{\mathbb{R}}
\newcommand{\bS}{\mathbb{S}}
\newcommand{\bT}{\mathbb{T}}
\newcommand{\bU}{\mathbb{U}}
\newcommand{\bV}{\mathbb{V}}
\newcommand{\bW}{\mathbb{W}}
\newcommand{\bX}{\mathbb{X}}
\newcommand{\bY}{\mathbb{Y}}
\newcommand{\bZ}{\mathbb{Z}}

\newcommand{\cA}{\mathcal{A}}
\newcommand{\cB}{\mathcal{B}}
\newcommand{\cC}{\mathcal{C}}
\newcommand{\cD}{\mathcal{D}}
\newcommand{\cE}{\mathcal{E}}
\newcommand{\cF}{\mathcal{F}}
\newcommand{\cG}{\mathcal{G}}
\newcommand{\cH}{\mathcal{H}}
\newcommand{\cI}{\mathcal{I}}
\newcommand{\cJ}{\mathcal{J}}
\newcommand{\cK}{\mathcal{K}}
\newcommand{\cL}{\mathcal{L}}
\newcommand{\cM}{\mathcal{M}}
\newcommand{\cN}{\mathcal{N}}
\newcommand{\cO}{\mathcal{O}}
\newcommand{\cP}{\mathcal{P}}
\newcommand{\cQ}{\mathcal{Q}}
\newcommand{\cR}{\mathcal{R}}
\newcommand{\cS}{\mathcal{S}}
\newcommand{\cT}{\mathcal{T}}
\newcommand{\cU}{\mathcal{U}}
\newcommand{\cV}{\mathcal{V}}
\newcommand{\cW}{\mathcal{W}}
\newcommand{\cX}{\mathcal{X}}
\newcommand{\cY}{\mathcal{Y}}
\newcommand{\cZ}{\mathcal{Z}}

\newcommand{\fa}{\mathfrak{a}}
\newcommand{\fb}{\mathfrak{b}}
\newcommand{\fc}{\mathfrak{c}}
\newcommand{\fd}{\mathfrak{d}}
\newcommand{\fe}{\mathfrak{e}}
\newcommand{\ff}{\mathfrak{f}}
\newcommand{\fg}{\mathfrak{g}}
\newcommand{\fh}{\mathfrak{h}}

\newcommand{\fj}{\mathfrak{j}}
\newcommand{\fk}{\mathfrak{k}}
\newcommand{\fl}{\mathfrak{l}}
\newcommand{\fm}{\mathfrak{m}}
\newcommand{\fn}{\mathfrak{n}}
\newcommand{\fo}{\mathfrak{o}}
\newcommand{\fp}{\mathfrak{p}}
\newcommand{\fq}{\mathfrak{q}}
\newcommand{\fr}{\mathfrak{r}}
\newcommand{\fs}{\mathfrak{s}}
\newcommand{\ft}{\mathfrak{t}}
\newcommand{\fu}{\mathfrak{u}}
\newcommand{\fv}{\mathfrak{v}}
\newcommand{\fw}{\mathfrak{w}}
\newcommand{\fx}{\mathfrak{x}}
\newcommand{\fy}{\mathfrak{y}}
\newcommand{\fz}{\mathfrak{z}}

\newcommand{\fA}{\mathfrak{A}}
\newcommand{\fF}{\mathfrak{F}}

\newcommand{\red}[1]{\textcolor{red}{#1}}
\newcommand{\blue}[1]{\textcolor{blue}{#1}}
\newcommand{\di}{d_\mathrm{{in}}}
\newcommand{\NN}{\textbf{\textrm{NN}}}

\begin{center}
    \textsuperscript{1}Department of Applied Mathematics, Fukuoka University\\
    \vspace{0.3cm}
    \textsuperscript{2}Department of Mathematics, University of Tsukuba\\
    \vspace{0.3cm}
    \textsuperscript{3}Independent Researcher
\end{center}

\begin{abstract}
We give a proof that, under relatively mild conditions, fully-connected feed-forward deep random neural networks converge to a \textit{Gaussian mixture distribution} as only the width of the last hidden layer goes to infinity. We conducted experiments for a simple model which supports our result. Moreover, it gives a detailed description of the convergence, namely, the growth of the last hidden layer gets the distribution closer to the Gaussian mixture, and the other layer successively get the Gaussian mixture closer to the normal distribution.
\end{abstract}

\section{Introduction}

Neural networks with a large number of parameters have had great success in recent years. However, their theoretical characteristics are not well understood yet. One direction to study the theoretical properties of neural networks is to take the limit of the number of parameters to infinity. In the following, we call neural networks with a very large number of parameters \textit{wide-width neural networks}. 

One of the known properties of wide-width neural networks is that fully-connected feed-forward random neural networks converge weakly to a Gaussian process as the widths -- the number of neurons -- of the hidden layers tend to infinity. The first study to show this phenomenon is \cite{Neal96}, in which he showed that neural networks with one hidden layer converge weakly to a Gaussian process as the width of the hidden layer goes to infinity.

Recently, it has been claimed that, under certain conditions, the wide-width neural networks with more than one hidden layer also converge weakly to a Gaussian process \cite{Matthews18,Lee18}. Following these works, several studies have given proofs of weak convergence to Gaussian processes in different ways \cite{Matthews18v2,Hanin21,Bracale21}\footnote{Note that \cite{Matthews18v2} is the extended version of \cite{Matthews18}. Although both have the same 
title, they differ greatly in content, including the main proof procedure.}. Our current research is also in the vein of this research direction. Although there are several works that touch on the convergence of  wide-width neural networks to Gaussian processes, giving the rigorous proof is a challenging task, as \cite{Matthews18} points out.

In the present paper, we discuss the weak convergence of wide-width  neural neural networks as only the width of the last hidden layer goes to infinity. Our main theorem is the following. We refer to Theorem \ref{thm:main} for more precise statement. 
\begin{theorem}\label{mainthm}
Consider a neural network with random weights and random biases. Suppose that the weights and biases has finite $n$-th moments for any $n \in \bN$. We also suppose that the activation function is bounded by a  polynomial. Then the neural network converges to a \textit{ Gaussian mixture distribution} as the width of the last hidden layer goes to infinity.
\end{theorem}
Here \textit{a Gaussian mixture distribution} means a mixture of centered Gaussian distributions \footnote{For notational simplicity, we call this distribution \textit{centered Gaussian mixture}.}(see also Definition \ref{def:mixt}). The main idea of the proof is to use the exchangeablity property of neural networks as random variables, inspired by \cite{Matthews18v2}. We give the following remarks on our main result:

\begin{enumerate}
\item While in the previous studies, it was considered the case that widths of all the hidden layers go to infinity, in the present paper only the width of the last hidden layer goes to infinity. 
This is the main difference between our main result and previous studies, and we consider it remarkable that only the last hidden layer limit makes the distribution Gaussian-like. 
\item As pointed out in Remark \ref{remark:fin-var}, the assumption that appears in our main result (Theorem \ref{thm:main}) is very mild and it does not seem possible to be removed. Moreover, our assumption is milder than those appearing in most previous studies. For example, in the paper \cite{Matthews18v2}, they assumed that weights and biases are sampled from Gaussian and the activation function is bounded by a linear function, which is pivotal to prove their main result. On the other hand, in our setting, we can consider more general weights, biases, 
\end{enumerate}

The remained part of this paper is organized as follows. 
In Section \ref{section:related-works}, we survey several researches of wide-width neural networks. 
In Section \ref{section:experiment}, we empirically confirm the validity of our result in a simple model. We verify that the neural network approaches the normal distribution as only the width of the last hidden layer goes to infinity by using \textit{ kernel two-sample test} method  \cite{Gretton12, Fukumizu07}. On the other hand, we verify that the approaching is not exactly a convergence to the normal distribution, which supports our statement that the limit is just a centered Gaussian mixture. For this, we compute a specific covariance of components of the neural network. We have a supplemental material section for mathematical details.

\subsubsection*{Acknowledgements}

We thank Jumpei Nagase for many assistances and helpful advices.
The first and the second authors are grateful to RIKEN AIP for good treatment as special postdoctral researchers.

\section{Related works}\label{section:related-works}

The seminal work that discusses the convergence of wide-width neural networks to Gaussian processes is done by Matthews et al. \cite{Matthews18}. They prove that when each layer grows at a particular rate respectively, neural networks with ReLU nonlinearity converge weakly to Gaussian processes. 
At about the same time, Lee et al. also give an insight on the infinitely wide random neural network, while their proof is not rigorous in that it seems conflating almost everywhere convergence and weak convergence \cite{Lee18}.
Although \cite{Matthews18} gives a mathematically rigorous proof, the condition for the proof is strict and somewhat artificial. Thus, several follow-up researches have been conducted to relax the conditions. The work of \cite{Matthews18v2} introduces an idea to use exchangeable central limit theorem for removing these constraints, which inspired us to study this subject. They prove the weak convergence to a Gaussian process under a condition that every covariance of squared pre-activations converges to 0.  
In contrast to \cite{Matthews18v2}, \cite{Bracale21} uses characteristic functions for the proof. Their proof assumes  weaker assumptions than \cite{Matthews18v2} in that it requires only polynomial envelop condition for an activation function, which is also the case with us, while they consider specific speed of growth of the widths of layers. The work of \cite{Hanin21} provides a strong result. Not only it requires activation function just a condition on its almost everywhere derivation, but also it requires the weights $w$ and biases $b$ have finite moments, which is also the case with us.

Although it is out of the scope of our work because we focus on fully-connected feed-forward neural networks, some studies discuss the extension of the relationship between neural networks and the Gaussian process beyond such neural networks. Some of them extend the proof to convolutional neural networks \cite{Novak19,Garriga-Alonso18}, a wide class of neural network architectures \cite{Yang19}, neural networks with bottleneck \cite{Agrawal20}, stable distribution \cite{Favaro21}, polynomial networks \cite{Klukowski21}, and uncountable inputs \cite{Bracale21}. 

\section{Main result}
\label{section:proof}
In this section, we prove Theorem \ref{mainthm} which is restated in Theorem \ref{thm:main} in a more precise manner.
Furthermore, in Corollary \ref{cor:equiv}, we give a sufficient (which is almost necessary) condition for the convergent distribution to be normal. We experimentally see that this condition seems hard to be attained, namely, the limit is not genuine Gaussian, in Section \ref{section:experiment}. 

For the proof, we use the notion of exchangeable sequence and de Finetti's theorem inspired by \cite{Matthews18v2}. We first give a short preliminary for the exchangeable sequence in subsection 3.1, and then we prove a central limit type theorem for the exchangeable sequence in a suitable setting for our neural network study in subsection 3.2. Finally, we prove the main theorem in subsection 3.3. We also give a short preliminary for measure-theoretic probability theory and exchangeable sequences in the supplemental material section. Throughout this paper, we write $\Omega$ for a sample space. 
We put $\bN := \{1, 2, 3, \dots \}$. 

\subsection{de Finetti's theorem}

\begin{definition}
We say that a sequence $X_{1}, X_{2}, X_{3},\dots$
of $\bR$-valued random variables is \textit{exchangeable} if for any  integer $N>0$
and any permutation $\tau \colon \{1, \ldots, N\} \longrightarrow \{1, \ldots, N\}$,  the joint probability distribution of the permuted sequence 
\[
X_{\tau (1)},  \dots, X_{\tau (N)}
\]
is the same as the joint probability distribution of the original sequence $X_{1},  \dots, X_{N}$. 
\end{definition}

The next lemma follows immediately from the definition of exchangeability.

\begin{lemma}\label{lemma:exchangeable}
Let $X_{1}, X_{2}, X_{3}, \dots$ and  $Y_{1}, Y_{2}, Y_{3},\dots$ be exchangeable sequences of $\bR$-valued random variables. 
\begin{itemize}
\item[(1)] The sequences $X_{1}+Y_{1}, X_{2}+Y_{2}, X_{3}+Y_{3},\dots$ and  $X_{1}Y_{1}, X_{2}Y_{2}, X_{3}Y_{3}, \dots$ are exchangeable. 
\item[(2)] The sequences $f(X_{1}), f(X_{2}), f(X_{3}), \dots$ is exchangeable for any measurable map $f \colon \bR \longrightarrow \bR$. 
\end{itemize}
\end{lemma}

\begin{definition}\label{dfdist}\ 
\begin{itemize}
\item[(1)] A function $F \colon \bR \longrightarrow [0,1]$ is called a \textit{(one-dimensional) distribution function} if $F$ is a  right continuous monotone increasing function satisfying
\[
\lim_{x \to -\infty}F(x) = 0, \,\,\, \textrm{ and } \,\,\, \lim_{x \to +\infty}F(x) = 1. 
\]
\item[(2)] We denote by $\fF$ the set of one-dimensional distribution functions. 
\item[(3)] We denote by $\mathfrak{A}$ the $\sigma$-field on $\fF$ generated by the class of sets 
$\fF(x,y) := \{F \in \fF \mid F(x) \leq y\}$. 
\item[(4)] For any  function $F \in \fF$, we write $X_{F}, X_{F,1}, X_{F,2}, X_{F,3}, \dots$ for i.i.d random variables with distribution $F$. 
\end{itemize}
\end{definition}

The following theorem is proved by De Finetti in \cite{Finetti37} (see also \cite[\S1]{BCRT58}). 
\begin{theorem}\label{thm:finetti}
Let $X_{1}, X_{2}, X_{3},\dots$ be an exchangeable sequence of $\bR$-valued random variables. 
Then there exists a probability measure $\mu_{X}$ (depending on  $X_{1}, X_{2}, X_{3},\dots$) on $(\fF, \fA)$ such that 
for any integer $n > 0$ and any measurable set $B \subset \bR^{n}$, we have 
\[
\bP((X_{1}, \dots, X_{n}) \in B) = \int_{\fF} \bP((X_{F, 1}, \dots, X_{F, n}) \in B) \, {d}\mu_{X}(F).  
\]
\end{theorem}

Theorem \ref{thm:finetti} and the definition of integration implies the following.
\begin{corollary}\label{cordefi}
Let $f : \bR^n \longrightarrow \bR$ be an integrable function with respect to the joint distribution of $X_1, \dots, X_n$. Then we have 
\[
\bE[f(X_{1}, \dots, X_{n})] = \int_{\fF} \bE[f(X_{F, 1}, \dots, X_{F, n})] \, {d}\mu_{X}(F).  
\]
\end{corollary}

In the following, let $X_{1}, X_{2}, X_{3},\dots$ be an exchangeable sequence of $\bR$-valued  random variables and let $\mu_{X}$ be the probability measure on $(\fF, \fA)$ whose existence is guaranteed by Theorem \ref{thm:finetti}. 

\begin{lemma}\label{lemma:third}
If $\bE[|X_{1}|^{r}] < \infty$, then 
\[
\mu_{X}(\{F \in \fF \mid \bE[|X_{F}|^{r}]  = \infty \}) = 0. 
\]
Here the set $\{F \in \fF \mid \bE[|X_{F}|^{r}]  = \infty \}$ is measurable by Proposition \ref{expmeasurable} in the supplemental material.
\end{lemma}
\begin{proof}
By Corollary \ref{cordefi} and the assumption, we have
\[
\int_{\fF}  \bE[|X_{F}|^{r}]  \, d\mu_{X}(F) =  \bE[|X_{1}|^{r}] < \infty. 
\]
Hence the statement follows.
\end{proof}

\begin{lemma}\label{lemma:zero}
We put $\fF_{0} := \{F \in \fF \mid \bE[X_{F}] = 0\}$, which is measurable by Proposition \ref{expmeasurable}. If $\bE[X_{1}X_{2}] = 0$, then we have 
\[
\mu_{X}(\fF \setminus \fF_{0}) = 0. 
\]
\end{lemma}
\begin{proof}
By Corollary \ref{cordefi}, we have $\bE[X_{1}X_{2}] = \int_{\fF} \bE[X_{F,1}X_{F,2}] \, d\mu_{X}(F)$. Since $X_{F,1}$ and $X_{F,2}$ are independent, we have $\int_{\fF} \bE[X_{F,1}X_{F,2}] \, d\mu_{X}(F) = \int_{\fF} \bE[X_{F}]^2 \, d\mu_{X}(F)$. Hence we obtain 
\[
0 = \bE[X_{1}X_{2}] = \int_{\fF} \bE[X_{F}]^{2} \, d\mu_{X}(F), 
\]
which implies that $\mu_{X}(\fF \setminus \fF_{0}) = 0$. 
\end{proof}

\subsection{Gaussian mixture distribution and the convergence theorem for exchangeable sequences}

\begin{definition}[Gaussian mixture distribution]\label{def:mixt} 
Let $\mathcal{N}$ be the set of probability density functions of all normal distributions. Let $\theta : \Omega \longrightarrow \mathcal{N}$ be any map, and consider a map $\tilde{\theta} : \Omega \times \bR \longrightarrow \bR_{\geq 0} ; \tilde{\theta}(\omega, x) = \theta(\omega)(x)$. Then a function $f_\theta$ defined by $f_\theta(x) = \int_\Omega \tilde{\theta}(\omega, x)$ is a probability density function, which we call a \textit{Gaussian mixture distribution function}.
\end{definition}

\begin{example}
The figures in Fig. \ref{fig:gauss_mix} show Gaussian mixture distribution of two centered Gaussians (left) and two non-centered ones (right).
\begin{figure}[H]
    \centering
    \includegraphics[width=0.48\linewidth]{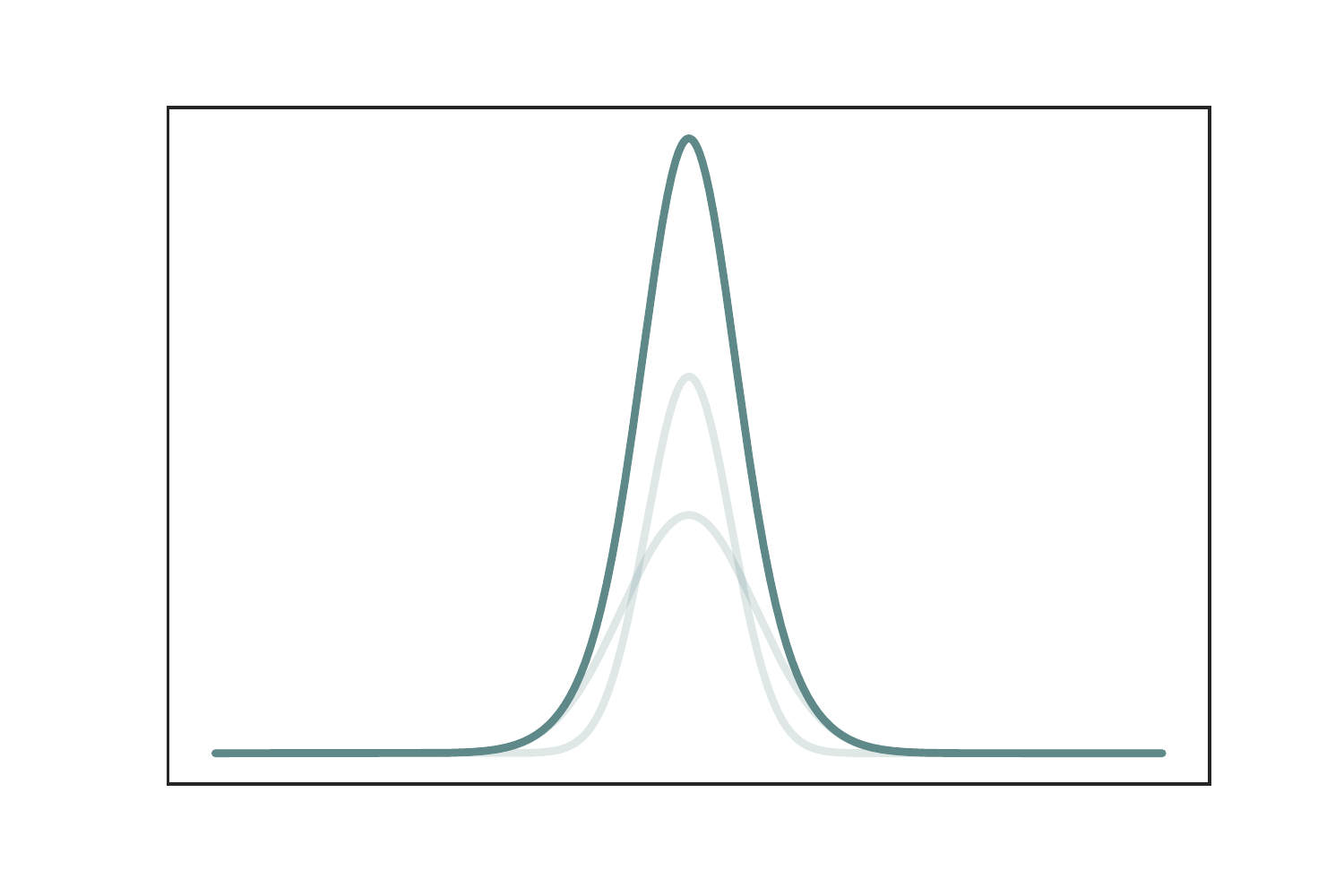}
    \includegraphics[width=0.48\linewidth]{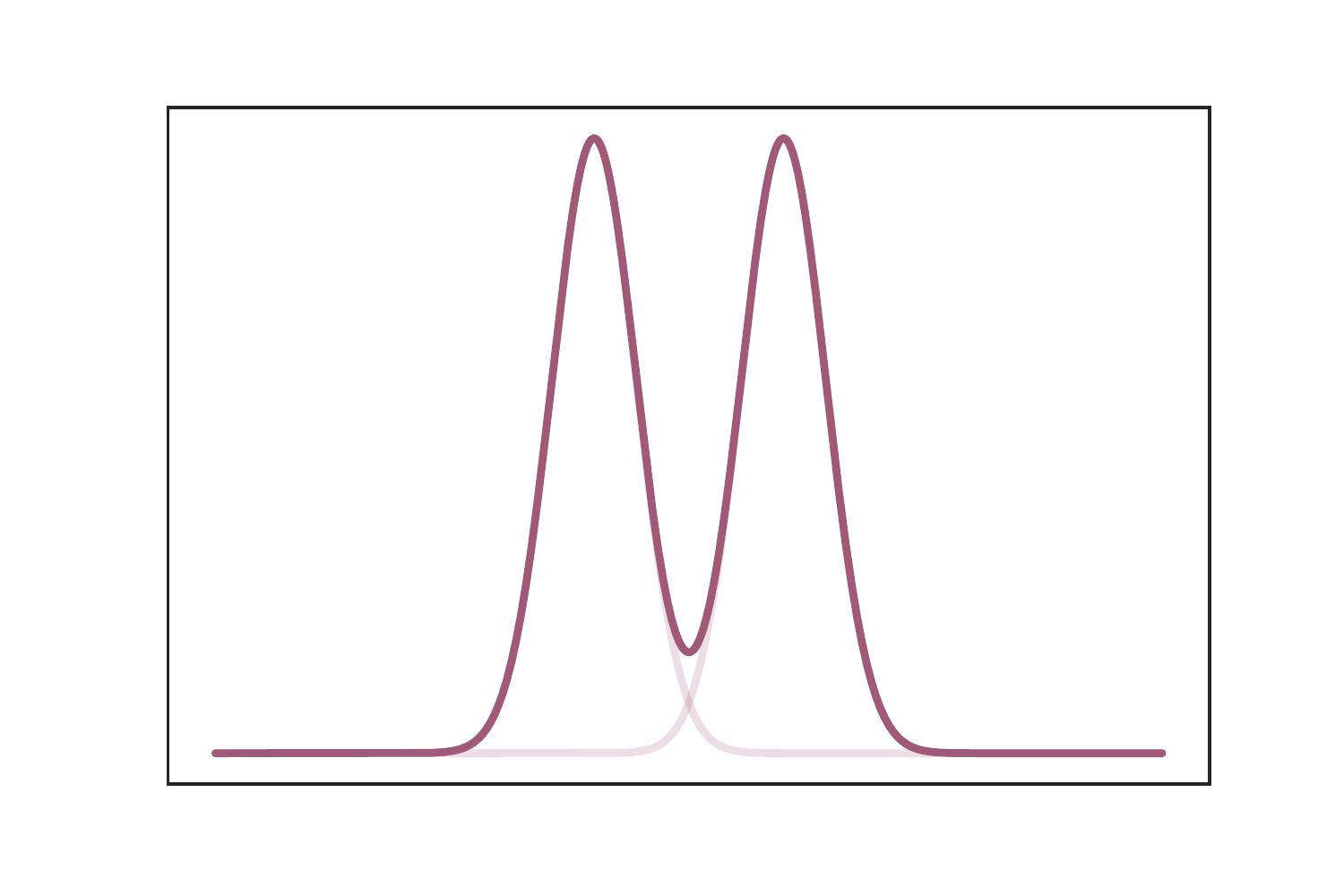}
    \caption{Centered vs Non-centered}
    \label{fig:gauss_mix}
\end{figure}
\end{example}

Now we construct a Gaussian mixture distribution function pivotal for our study of wide-width neural networks. Let $X_{i}$ and $\mu_{X}$ be as in the previous subsection. 
Suppose that $\bE[X_{1}] = 0$, $\bE[X_{1}X_{2}] = 0$, and $\bE[|X_{1}|^{2}] < \infty$. 
We put 
\[
\fF_{0}' := \{F \in \fF \mid \bE[X_{F}] = 0 \, \textrm{ and } \, \bE[|X_{F}|^{2}]  < \infty \}, 
\]
which is measurable by Proposition \ref{expmeasurable} in the supplemental material. Then by Lemmas \ref{lemma:third} and \ref{lemma:zero}, we have 
\[
\mu_{X}(\fF \setminus \fF_{0}') = 0. 
\]

\begin{definition}\label{cgauss}
We define a probability density function $f_{X}(x)$ by 
\begin{align*}
f_{X}(x) &:= \int_{\fF_{0}'} \frac{1}{\sqrt{2\pi \bE[|X_{F}|^{2}] }}\exp\left(-\frac{x^{2}}{2\bE[|X_{F}|^{2}] }\right) \, d\mu_{X}(F). 
\end{align*}
Let $Z_{X}$ denotes a random variable with the probability density function $f_{X}(x)$ and 
we put its characteristic function  $\Phi_{X}(t) := \bE[e^{itZ_{X}}]$. 
\end{definition}
\begin{remark}
The distribution in Definition \ref{cgauss} is a mixture of centered Gaussian (namely 0 mean), which we call {\it centered Gaussian mixture}. Hence it is more like normal distribution. 
\end{remark}

\begin{lemma}\label{lemma:phi}
We have 
\[
\Phi_{X}(t) = \int_{\fF_{0}'}\mathrm{exp}\left(-\frac{\bE[|X_{F}|^{2}] t^{2}}{2}\right) \, d\mu_{X}(F). 
\]
\end{lemma}
\begin{proof}
By definition, we have 
\begin{align*}
\Phi_{X}(t) &= \bE[e^{itZ_{X}}] 
\\
&= \int_{\bR} e^{itx}f_{X}(x) \, dx 
\\
&= \int_{\bR} \int_{\fF_{0}'} e^{itx} \frac{1}{\sqrt{2\pi \bE[|X_{F}|^{2}]}}\exp\left(-\frac{x^{2}}{2 \bE[|X_{F}|^{2}]}\right) \, d\mu_{X}(F) \, dx. 
\end{align*}
Since 
\begin{align*}
&\int_{\fF_{0}'} \int_{\bR}\left| e^{itx} \frac{1}{\sqrt{2\pi\rho_{2}(X_{F})}}\exp\left(-\frac{x^{2}}{2\rho_{2}(X_{F})}\right) \right|  \, dx \, d\mu_{X}(F) 
\\
&= \int_{\fF_{0}'} \int_{\bR}\frac{1}{\sqrt{2\pi \bE[|X_{F}|^{2}]}}\exp\left(-\frac{x^{2}}{2 \bE[|X_{F}|^{2}]}\right)  \, dx \, d\mu_{X}(F)
\\
&= \int_{\fF_{0}'}1 \, d\mu_{X}(F) 
\\
&=1 < \infty, 
\end{align*}
Fubini--Tonelli theorem implies that 
\begin{align*}
 \Phi_{X}(t) &= \int_{\bR} \int_{\fF_{0}'} e^{itx} \frac{1}{\sqrt{2\pi \bE[|X_{F}|^{2}]}}\exp\left(-\frac{x^{2}}{2 \bE[|X_{F}|^{2}]}\right) \, d\mu_{X}(F) \, dx 
 \\
 &=   \int_{\fF_{0}'}  \int_{\bR} e^{itx} \frac{1}{\sqrt{2\pi \bE[|X_{F}|^{2}]}}\exp\left(-\frac{x^{2}}{2 \bE[|X_{F}|^{2}]}\right) \, dx \, d\mu_{X}(F)
 \\
 &= \int_{\fF_{0}'} \mathrm{exp}\left(-\frac{ \bE[|X_{F}|^{2}]t^{2}}{2}\right) \, d\mu_{X}(F). 
 \end{align*}
 Here the last equality follows from the fact that the characteristic function of the Gaussian distribution with mean $0$  and standard deviation $\sigma$ is given by $\mathrm{exp}\left(-\sigma^{2}t^2/2\right)$. 
\end{proof}

\begin{theorem}\label{thm:exchan-central}
Suppose that $\bE[X_{1}] = 0$, $\bE[X_{1}X_{2}] = 0$, and $\bE[X_{1}^{2}]  < \infty$. 
Let 
\[
S_{n} := \frac{1}{\sqrt{n}}\sum_{i=1}^{n}X_{i}. 
\]
Then the sum $S_{n}$ of random variables converges in distribution to the centered  Gaussian mixture random variable $Z_{X}$. 
\end{theorem}
\begin{proof}
We take a real number $t \in \bR$. 
Let $f_{n}(t) := \bE[e^{itS_{n}}]$ denote the characteristic function of $S_{n}$. 
Then we have 
\begin{align*}
f_{n}(t)  = \int_{\fF_{0}'} \bE[e^{itS_{F, n}}] \, d\mu_{X}(F), 
\end{align*}
where $S_{F, n} := \frac{1}{\sqrt{n}}\sum_{i=1}^{n}X_{F, i}$. 
Since $|\bE[e^{itS_{F, n}}] | \leq 1$ and $\mu_{X}(\fF_{0}') = 1$,  Lebesgue's dominated convergence theorem  implies that 
\[
\lim_{n \to \infty} f_{n}(t)  = \int_{\fF_{0}'} \lim_{n \to \infty}\bE[e^{itS_{F, n}}] \, d\mu_{X}(F). 
\]
By the central limit theorem, for each function $F \in \fF_{0}'$, we have 
\[
\lim_{n \to \infty}\bE[e^{itS_{F, n}}] = \mathrm{exp}\left(-\frac{\bE[X_{F}^{2}]t^{2}}{2}\right), 
\]
and hence $\lim_{n \to \infty} f_{n}(t)  = \Phi_{X}(t)$. 
L\'evy's continuity theorem shows that $S_{n}$ converges in distribution to $Z_{X}$. 
\end{proof}

The following gives a sufficient condition for the distribution of $Z_X$ to be normal.
\begin{corollary}\label{cor:equiv}
Suppose that $\bE[X_{1}] = 0$, $\bE[X_{1}X_{2}] = 0$, and $E[|X_{1}|^{4}] < \infty$. 
Then the following are equivalent. 
\begin{itemize}
\item[(1)] $\mathrm{Cov}(X_{1}^{2}, X_{2}^{2}) = 0$. 
\item[(2)] $\mu_{X}(\{F \in \fF_{0} \mid \bE[|X_{F}|^{2}] = \bE[|X_{1}|^{2}]\}) = 1$. 
\item[(3)] The random variable $Z_{X}$ is normally distributed. 
\end{itemize}
\end{corollary}
\begin{proof}
It is proved in the paper \cite{BCRT58} that claims (1) and (2) are equivalent. 
It follows immediately from the definition of $Z_{X}$ that claim (2)  implies claim (3). 

Let us prove that claim (3)  implies claim (1). 
Put $\fF_{0}'' := \{F \in \fF \mid \bE[X_{F}] = 0 \, \textrm{ and } \, \bE[|X_{1}|^{4}]  < \infty \}$. 
Since we assume that  $E[|X_{1}|^{4}] < \infty$,  we have $\mu_{X}(\fF_{0}' \setminus \fF_{0}'') = 0$ by Lemma \ref{lemma:third}. We also have $\bE[|X_{F}|^{2}] = \bE[|X_{1}|^{2}]$ and $\bE[|X_{F}|^{2}]^{2} = \bE[|X_{1}X_{2}|^{2}]$ by Corollary \ref{cordefi} and the independence of $X_F$'s.
Hence  Lemma \ref{lemma:phi} shows that 
\begin{align*}
\Phi_{X}(t) &= \int_{\fF_{0}''} 1 - \frac{ \bE[|X_{F}|^{2}]t^{2}}{2} + \frac{ \bE[|X_{F}|^{2}]^{2}t^{4}}{8} + O(t^{6}) \, d\mu_{X}
\\
&= 1 - \frac{ \bE[|X_{1}|^{2}]t^{2}}{2} + \frac{3\bE[|X_{1}X_{2}|^{2}]t^{4}}{4!} + O(t^{6}), 
\end{align*}
which implies that $\bE[Z_{X}] = 0$, $\bE[Z_{X}^{2}] = \bE[X_{1}^{2}]$, and $\bE[Z_{X}^{4}] = 3\bE[|X_{1}X_{2}|^{2}]$. 
Since we suppose that  $Z_{X}$ is normally distributed, we have 
\[
\Phi_{X}(t) = \mathrm{exp}\left(-\frac{ \bE[|X_{1}|^{2}]t^{2}}{2}\right). 
\]
This fact implies that $3  \bE[X_{1}^{2}]^{2} = \bE[Z_{X}^{4}]  = 3\bE[|X_{1}X_{2}|^{2}]$, i.e., 
$\mathrm{Cov}(X_{1}^{2}, X_{2}^{2}) = \bE[X_{1}^{2}X_{2}^{2}] - \bE[X_{1}^{2}] \bE[X_{2}^{2}] = (\bE[Z_{X}^{4}] - \bE[Z_{X}^{4}])/3 = 0$. 
\end{proof}

\subsection{Convergence theorem for wide-width neural networks}\label{subsec:conv-thm}

\subsubsection*{{\bf Settings}}
Let us consider a neural network with $L$ hidden layers and an activation function $\sigma : \bR \longrightarrow \bR$ which is a measurable map. 
That is a sequence of maps $\sigma \circ g_0 : \bR^{d_0} \longrightarrow \bR^{d_1}, \sigma \circ g_1 : \bR^{d_1} \longrightarrow \bR^{d_2}, \dots, \sigma \circ g_{L-1} : \bR^{d_{L-1}} \longrightarrow \bR^{d_{L}}$ and $g_{L} : \bR^{d_{L}} \longrightarrow \bR^{d_{L+1}}$ (see Definition \ref{def:nn} for details). Here we denote by $\sigma$ the $d_i$ times direct product of $\sigma$ for each $i$, and $d_0, d_{L+1}$ are input, output dimensions respectively. 

In the previous research of wide-width neural networks, one considers a limit of the neural network as $d_1, \dots, d_{L} \to \infty$. Instead of dealing with this limit literally, in the present paper,  we extend the domains of maps $g_i$'s and $\sigma$'s to $\bR^\bN$, and take a limit of their supports. This modification may not have any discrepancy in the setting of previous studies of wide-width neural networks.

For any $\ell \in \{0, \dots, L\}$ and any positive integers $i$ and $j$, we take random variables 
\[
w^{(\ell)}_{i,j} \colon \Omega \longrightarrow \bR \,\,\,  \textrm{ and } \,\,\, b^{(\ell)}_{i} \colon \Omega \longrightarrow \bR
\]
satisfying the following: 
\begin{itemize}
\item[(a)] The set $\{w^{(\ell)}_{i,j}, b^{(\ell)}_{i} \mid i,j \in \bN, 0 \leq \ell \leq L\}$ of the random variables  are mutually independent. 
\item[(b)] For any  $0 \leq \ell \leq L$ and $j \geq 1$, the random variables $w^{(\ell)}_{1, j}, w^{(\ell)}_{2, j}, w^{(\ell)}_{3, j}, \dots$ are identically distributed. 
\item[(c)] For any  $0 \leq \ell \leq L$, the random variables $b^{(\ell)}_{1}, b^{(\ell)}_{2}, b^{(\ell)}_{3}, \dots$ are identically distributed. 
\item[(d)] For any $i, j \in \bN$, we have $\bE[w^{(L)}_{i, j}] = 0$.  
\end{itemize}

\begin{remark}
In previous studies \cite{Matthews18,Matthews18v2}, the random variables $w^{(\ell)}_{i,j}$ and $b^{(\ell)}_{i}$ are assumed to be normally distributed. 
In the present paper, however, we are dealing with a more general case. 
\end{remark}

 For notational simplicity, we put 
 \[
\bm{b}^{(\ell)} := (b_i^{(\ell)})_{i > 0} \colon \Omega \longrightarrow \bR^{\bN}. 
 \]

\begin{definition}[Neural networks]\label{def:nn}\ 
\begin{itemize}
\item[(1)] 
We define a map $\widetilde{\NN}^{(0)} \colon \Omega \times \bR^{d_{0}} \longrightarrow \bR^{\bN}$ by 
\[
\mathrm{pr}_{j} (\widetilde{\NN}^{(0)}(\omega, t_{1}, \dots, t_{d_{0}})) := \frac{1}{\sqrt{d_{0}}} \sum_{i=1}^{d_{0}} t_{i} w^{(0)}_{i,j}(\omega). 
\]
We also define a map $\NN^{(0)} \colon \Omega \times \bR^{d_{0}} \longrightarrow \bR^{\bN}$ by 
\[
\NN^{(0)}(\omega, t) := \sigma_{0}(\widetilde{\mathrm{NN}}^{(0)}(\omega, t) + \bm{b}^{(0)}(\omega)). 
\]

\item[(2)] Let $d \geq 1$ be an integer and $1 \leq \ell \leq L$. 
We define a map $\widetilde{\NN}^{(\ell)}_{d} \colon \Omega \times \bR^{\bN} \longrightarrow \bR^{\bN}$ by 
\[
\mathrm{pr}_{j}(\widetilde{\NN}^{(\ell)}_{d}(\omega, (t_{i})_{i \in \bN})) = \frac{1}{\sqrt{d}}\sum_{i=1}^{d}t_{i}w^{(\ell)}_{i,j}(\omega).  
\]

\item[(3)] For each integer $0 \leq \ell \leq L-1$, we define a map $\NN^{(d_{0}, \dots, d_{j})} \colon \Omega \times \bR^{d_{0}}  \longrightarrow \bR^{\bN}$ by 
\begin{itemize}
\item $\NN^{(d_{0})}(\omega, -) := \widetilde{{\NN}}^{(0)}(\omega, -)$, 
\item 
$
\NN^{(d_{0}, \dots, d_{\ell +1})}(\omega, -) := \widetilde{\NN}^{(\ell+1)}_{d_{\ell + 1}}(\omega, -) \circ \sigma(\NN^{(d_{1}, \dots, d_{\ell})}(\omega, -) + \bm{b}^{(\ell)}(\omega)). 
$
\end{itemize}

\item[(4)]  We call $\NN^{(d_{0}, d_{1}, \dots, d_{L})} \colon \Omega \times \bR^{d_{0}}  \longrightarrow \bR^{\bN}
$ a $(d_{1}, \dots, d_{L})$-layer neural network  stochastic process. 
\end{itemize}
\end{definition}

\begin{remark}
We use the so-called \textit{NTK parametrization}; scaling factor $1/\sqrt{d_l}$ is multiplied to matrix product of activation and weight, instead of taking standard deviation of the weights proportional in $1/\sqrt{d_l}$ (standard parametrization) \cite{Jacot18}. Although the parametrization differs, both parametrization represent the same set of functions and this difference has no effect on prediction.
\end{remark}
\begin{lemma}\label{lemma:nnsp}
For any integer $0 \leq \ell \leq L$, the  neural network stochastic process  
$\NN^{(d_{0}, \dots, d_{\ell})} \colon \Omega  \times \bR^{d_{0}}  \longrightarrow \bR^{\bN}$ is measurable. 
\end{lemma}
\begin{proof}
This lemma follows from the facts that the sum of measurable maps is measurable and that the composition of measurable maps is measurable. 
\end{proof}

\begin{definition}
For each real vector $t \in \bR^{d_0}$, we put 
\[
\NN^{(d_{0}, \dots, d_{L})}(-, t) =: (S_{1}^{(d_{0}, \dots, d_{L})}(t), S_{2}^{(d_{0}, \dots, d_{L})}(t), S_{3}^{(d_{0}, \dots, d_{L})}(t), \dots ). 
\]
Note that $S_{i}^{(d_{0}, \dots, d_{L})}(t)$ is a random variable by Lemma \ref{lemma:nnsp}. 
\end{definition}

Fix a real vector $t \in \bR^{d_0}$ and an integer $\ell \geq 1$. 
We put 
\begin{align*}
\sigma(\NN^{(d_{0}, \dots, d_{\ell})}(-, t) + \bm{b}^{(\ell)}(-)) =: (X_{1}^{(d_{0}, \dots, d_{\ell})}, X_{2}^{(d_{0}, \dots, d_{\ell})},  \dots). 
\end{align*}
Then by definition, we have 
\begin{align}\label{eq:S}
S_{j}^{(d_{0}, \dots, d_{L})}(t) = \frac{1}{\sqrt{d_{L}}}\sum_{i=1}^{d_{L}}X_{i}^{(d_{0}, \dots, d_{L-1})}w_{i,j}^{(L)}. 
\end{align}

For notational simplicity, we put $Y_{i} := X_{i}^{(d_{0}, \dots, d_{L-1})}w_{i,j}^{(L)}$, which is independent of $j$ in distribution. Then we have 
\[
S_{j}^{(d_{0}, \dots, d_{L})}(t) = \frac{1}{\sqrt{d_{L}}}\sum_{i=1}^{d_{L}}Y_i. 
\]

\begin{proposition}\label{prop:main}
Suppose that 
$\bE[Y_1^2] < \infty$. 
\begin{itemize}
\item[(1)] The sequence $Y_{1}, Y_{2}, Y_{3}, \dots$ of random variables is exchangeable. 
\item[(2)] $\bE[Y_{1}] = 0$. 
\item[(3)] $\bE[Y_{1}Y_{2}] = 0$. 
\item[(4)]$\bE[|S_{j}^{(d_{0}, \dots, d_{L})}(t)|^2] = \bE[Y_1^2]$.
\end{itemize}
\end{proposition}
\begin{proof}
Claim (1) follows from the definition of $Y_{i}$ and Lemma \ref{lemma:exchangeable}. 
Since we assume that $\bE[Y^2_1]<\infty$, claims (2) and (3) follow from easy direct computations; 
\begin{align*}
\bE[Y_{1}] &= \bE[X_{1}^{(d_{0}, \dots, d_{L-1})}] \bE[w_{i,j}^{(L)}] = 0, 
\\ 
\bE[Y_{1}Y_{2}] &= \bE[X_{1}^{(d_{0}, \dots, d_{L-1})}X_{2}^{(d_{0}, \dots, d_{L-1})}] \bE[w_{1,j}^{(L)}] \bE[w_{2,j}^{(L)}] = 0,
\end{align*} 
here we have 
\[
\bE[X_{1}^{(d_{0}, \dots, d_{L-1})}X_{2}^{(d_{0}, \dots, d_{L-1})}] 
\leq \frac{1}{2}\left(\bE[|X_{1}^{(d_{0}, \dots, d_{L-1})}|^2] + \bE[|X_{2}^{(d_{0}, \dots, d_{L-1})}|^2] \right)=  \bE[Y_1^2]/\bE[w_{i,j}^2]
< \infty. 
\]
Claim (4) follows from claim (3).
\end{proof}

\begin{corollary}\label{cor:main}
Suppose that $\bE[|Y_1|^2] < \infty$ and fix a real vector $t \in \bR^{d_0}$.
\begin{itemize}
\item[(1)] The sequence $S_{1}^{(d_{0}, \dots, d_{L})}(t), S_{2}^{(d_{0}, \dots, d_{L})}(t), S_{3}^{(d_{0}, \dots, d_{L})}(t), \dots$ of random variables is exchangeable. 
\item[(2)] $\bE[S_{1}^{(d_{0}, \dots, d_{L})}(t)] = 0$. 
\item[(3)] $\bE[S_{1}^{(d_{0}, \dots, d_{L})}(t)S_{2}^{(d_{0}, \dots, d_{L})}(t)] = 0$. 
\item[(4)] $\bE[|S_{j}^{(d_{0}, \dots, d_{L})}(t)|^2] < \infty$.
\end{itemize}
\end{corollary}

The following is the main result of the present paper.

\begin{theorem}[Convergence theorem]\label{thm:main}
Let $t \in \bR^{d_0}$ be a real vector and  $i \geq 1$ an  integer. 
If $\bE[|Y_1|^2] < \infty$, then the random variable $S_{i}^{(d_{0}, \dots, d_{L})}(t)$  converges weakly to a centered Gaussian mixture random variable as $d_{L} \to \infty$. 
\end{theorem}

\begin{proof}
This theorem follows from Theorem \ref{thm:exchan-central} and Corollary \ref{cor:main}. 
\end{proof}

\begin{remark}\label{remark:fin-var}
The finite variance condition always appears in any kind of central limit type theorems. 
(So we may not remove this condition.)  
\end{remark}

Note that the assumption that $\bE[Y^2_1]<\infty$ (appears in Theorem \ref{thm:main}) is very weak as shown in Proposition \ref{xinl} below. 
More precisely, when the map $\sigma$ is a typical activation function (e.g., Binary step, ReLU,... ,more generally, a function bounded by a polynomial function) and 
the $n$-th moments of $w_{i,j}^{(\ell)}$ and $b_j$ are finite for all $n$  (e.g., $w_{i,j}^{(\ell)}$ and $b_j$  are normally distributed), then $\bE[Y^2_1]<\infty$.

\begin{proposition}\label{xinl}
Suppose that any $w^{(\ell)}_{i, j}, b^{(\ell)}_j$ have finite $n$-th moments for any $n \in \bN$. Further, we suppose that the activation function $\sigma$ satisfies $|\sigma| \leq |u|$ for some polynomial $u$. Then we have $\bE[|Y_1|^2] < \infty$.
\end{proposition}
\begin{proof}
We prove it by induction on $\ell$. Since we have 
\[
X^{(d_0)}_1(t) = \sigma \left(\frac{1}{\sqrt{d_0}}\sum_{i = 1}^{d_0}t_i w^{(0)}_{i, 1} + b^{(0)}_1\right),
\]
we obtain that $X^{(d_0)}_1(t) \in \tilde{L}(\Omega)$ by Corollary \ref{prodl} in supplemental material. Further, we have 
\begin{align*}
X^{(d_0, \dots, d_{\ell+1})}_1(t) &= \sigma \left( \mathrm{ pr}_1(\NN^{(d_0, \dots, d_{\ell+1})})+ b^{(\ell + 1)}_1\right) \\
&= \sigma \left(\frac{1}{\sqrt{d_{\ell + 1}}}\sum_{i = 1}^{d_{\ell + 1}}X^{(d_0, \dots, d_{\ell})}_i(t)w^{(\ell + 1)}_{i, 1} + b^{(\ell + 1)}_1\right),
\end{align*}
we obtain that $X^{(d_0, \dots, d_{\ell+1})}_1(t) \in \tilde{L}(\Omega)$ by Corollary \ref{prodl} in supplemental material and the induction hypothesis. In particular, we obtain $\bE[|Y_1|^2] < \infty$.
\end{proof}

\subsection{Simple verification}

Figure \ref{fig:fig1} shows distributions of outputs for several width of $3$ hidden layers. We can see that the distribution approaches a Gaussian mixture when when $d_3$ gets large. 
Moreover, somewhat surprisingly, when $d_3$ is not sufficiently large, the distribution does not change much even if $d_1$ and $d_2$ get large. 
The most Gauss-like case may be when $d_1, d_2$ and $d_3$ gets large. 
\begin{figure}[H]\label{fig:fig1}
    \centering
    \includegraphics[width=\textwidth]{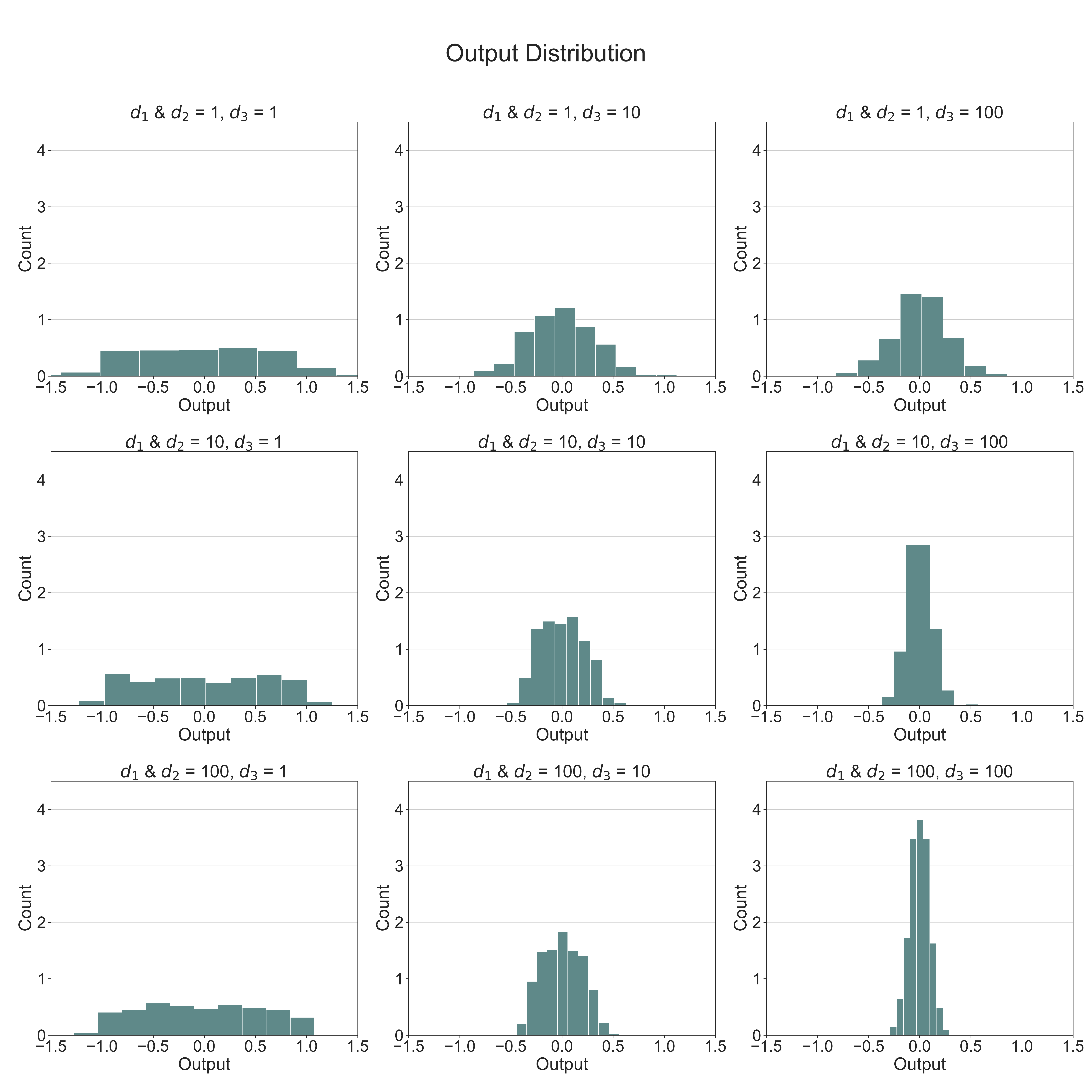}
    \caption{Output distribution of neural networks.}
    \label{fig:output_1_10_100}
\end{figure}

\section{Experiments}\label{section:experiment}
In this section, we empirically study the behavior of convergence of wide-width neural networks as the widths go to infinity. In subsection \ref{sub2}, we compute the difference between a neural network and the normal distribution. The computation result shows that the difference gets smaller as the width of the last hidden layer gets larger. Note that this difference is not necessarily getting to 0. In subsection \ref{sub3}, we verify that the neural network can converge to the normal distribution as widths of the layer other than the last one get large. This result supports our assertion that the difference considered in subsection \ref{sub2} is not necessarily 0. Further, the above experiments give a detailed description of the convergence, namely, the growth of the last hidden layer gets the distribution closer to the Gaussian mixture, and the other layer successively get the Gaussian mixture closer to the normal distribution.

\subsection{difference from the normal distribution}\label{sub2}
To quantify the difference between distributions, we employ the \textit{ kernel two-sample test} method from \cite{Gretton12, Fukumizu07}. We briefly recall this method in the supplemental material. By Proposition \ref{quantity} in the supplemental material, we calculate the \textit{maximum mean discrepancy}, that is the following quantity
\begin{equation}\label{dist3}
    \frac{1}{\sqrt{2\sigma^2 + 1}} - \frac{2}{n}\sum_{i=1}^{n}\frac{1}{\sqrt{\sigma^2 + 1}} \exp\left(\frac{y_i^2}{2(\sigma^2 + 1)}\right) + \frac{1}{n^2}\sum_{i, j = 1}^{n}\exp\left(-\frac{(y_i-y_j)^2}{2}\right),
\end{equation}
where $\sigma^2$ is the variance of the samples $\mathcal{Y} = \{y_n\}_n$ from the distribution of the neural network.
The neural network we deal with is in the following setting : 
\begin{itemize}
    \item Setting of the neural network
        \begin{itemize}
            \item Layers: input (1-dimensional), 3 hidden layers (width $d_1, d_2$ and $d_3$ respectively), output (1-dimensional), 
            \item activation function: ReLU, 
            \item input data: generated from the standard normal distribution $N(0, 1)$.
            \item weight initialization: we sample weights $w_{i, j}^{(l)}$ from the uniform distribution on a interval $U(-1/\sqrt{d_{L}}, 1/\sqrt{d_{L}})$, following Pytorch default initialization \cite{Paszke19}. 
       \end{itemize}
\end{itemize}
The computation results are in the Figure \ref{fig:mmd}.
\begin{figure}[h]
    \centering
    \includegraphics[width=0.8\textwidth]{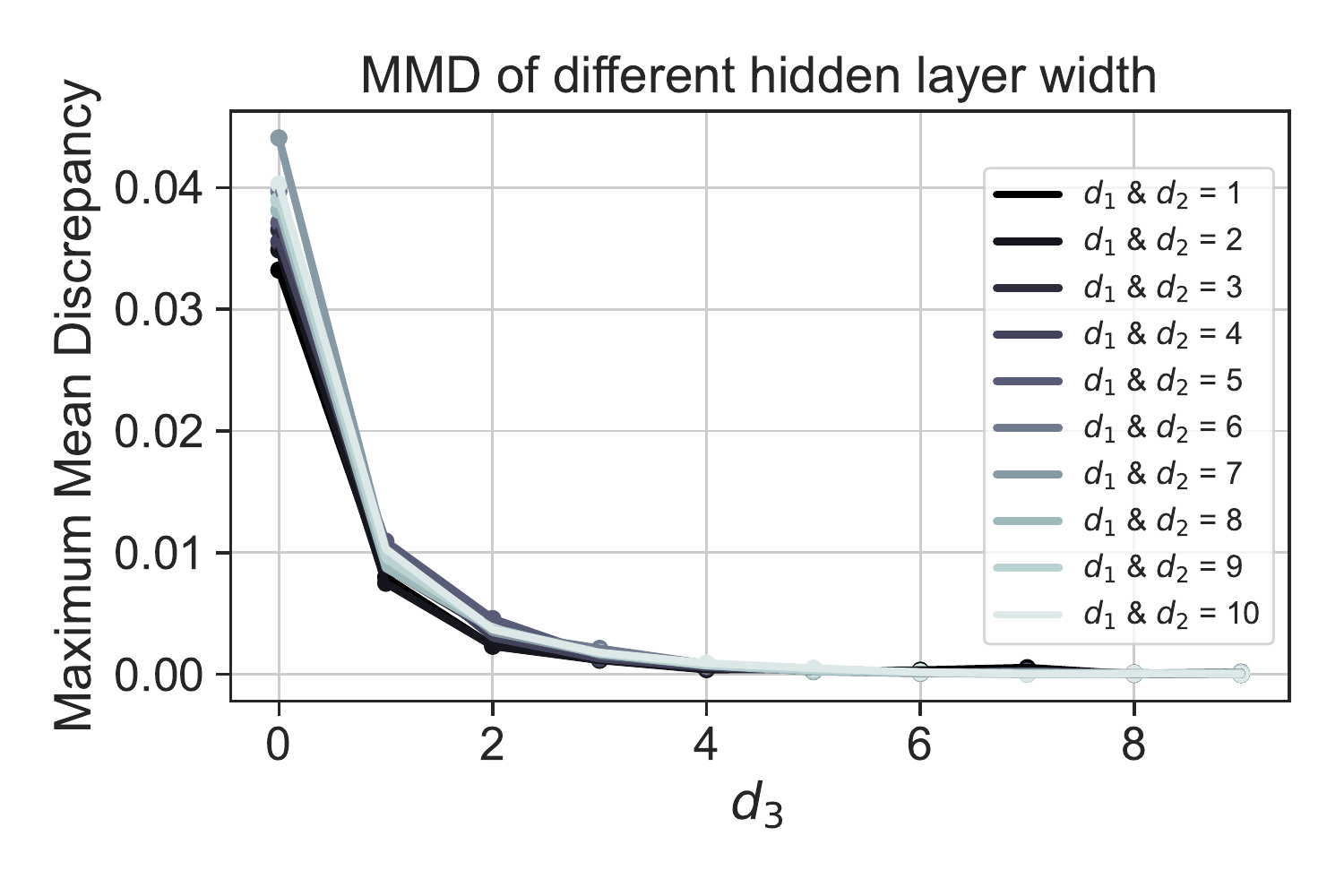}
    \caption{Maximum mean discrepancy between output distribution and Gaussian distribution}
    \label{fig:mmd}
\end{figure}

\underline{\bf observation :} In the above computation, we find that the quantity (\ref{dist3}) gets small, namely approaching the normal distribution, independently of $d_1, d_2$ as $d_3$ gets large.

\subsection{convergence to the normal distribution}\label{sub3}
We use the same notations as in \S\ref{subsec:conv-thm}. 
We have shown in Theorem \ref{thm:main} that  $(1/\sqrt{d_L}) \sum_{i=1}^{d_L}Y_i$ converges weakly to a Gaussian mixture distribution. 
Moreover, by Corollary \ref{cor:equiv}, we see that this Gaussian mixture distribution is Gaussian if and only if $\mathrm{Cov}(Y_1^2, Y_2^2) = 0$. 
In this subsection, we empirically compute the value $\mathrm{Cov}(Y_1^2, Y_2^2)$ in the same setting of neural network in subsection \ref{sub2}. 
Note that, in this setting, the value $\mathrm{Cov}(Y_1^2, Y_2^2)$ depends on the dimensions $d_1$ and $d_2$ of the first and second hidden layers. 
As realised in previous studies, the limit 
\[
\lim_{d_1, d_2, d_3 \to \infty}\frac{1}{\sqrt{d_L}} \sum_{i=1}^{d_L}Y_i
\]
converges weakly to a Gaussian distribution. 
Hence it should be happened that 
the Gaussian mixture distribution tends to a Gaussian distribution as $d_1$ and $d_2$ go to $\infty$, which means that 
\[
\lim_{d_1, d_2 \to \infty}\mathrm{Cov}(Y_1^2, Y_2^2) = 0. 
\]

To check how $\mathrm{Cov}(Y_1^2, Y_2^2)$ approaches $0$ when each hidden layer grows, we empirically compute the covariance of squared outputs of the last hidden layers. The setting of the neural network is the same as that of subsection \ref{sub2}. Unless otherwise noted, the width of the last hidden layer is $2$ and the widths of the other hidden layer are $1$. We sample neural networks $1000$ times and compute the covariance between $Y_1^2$ and $Y_2^2$ on these samples. We compare three cases, i) only $d_1$ grows, ii) only $d_2$ grows, iii) and both $d_1$ and $d_2$ grow, to see how each hidden layer width affects the result. 

We can see from the Figures \ref{fig:covariance-1}, \ref{fig:covariance-2}, and \ref{fig:covariance-12} that 
$\mathrm{Cov}(Y_1^2, Y_2^2)$ actually approaches 0 as $d_1$ and $d_2$ go to infinity, as expected. 
Moreover, for $\mathrm{Cov}(Y_1^2, Y_2^2)$ to approach 0, it is necessary that $d_1$ and $d_2$ go to infinity. 
Even if only one of $d_1$ and $d_2$ goes to infinity, $\mathrm{Cov}(Y_1^2, Y_2^2)$ does not approach 0. 
In particular, one can say that the Gaussian mixture distribution is not Gaussian (in general). 

\begin{figure}[H]
    \centering
    \includegraphics[width=0.8\textwidth]{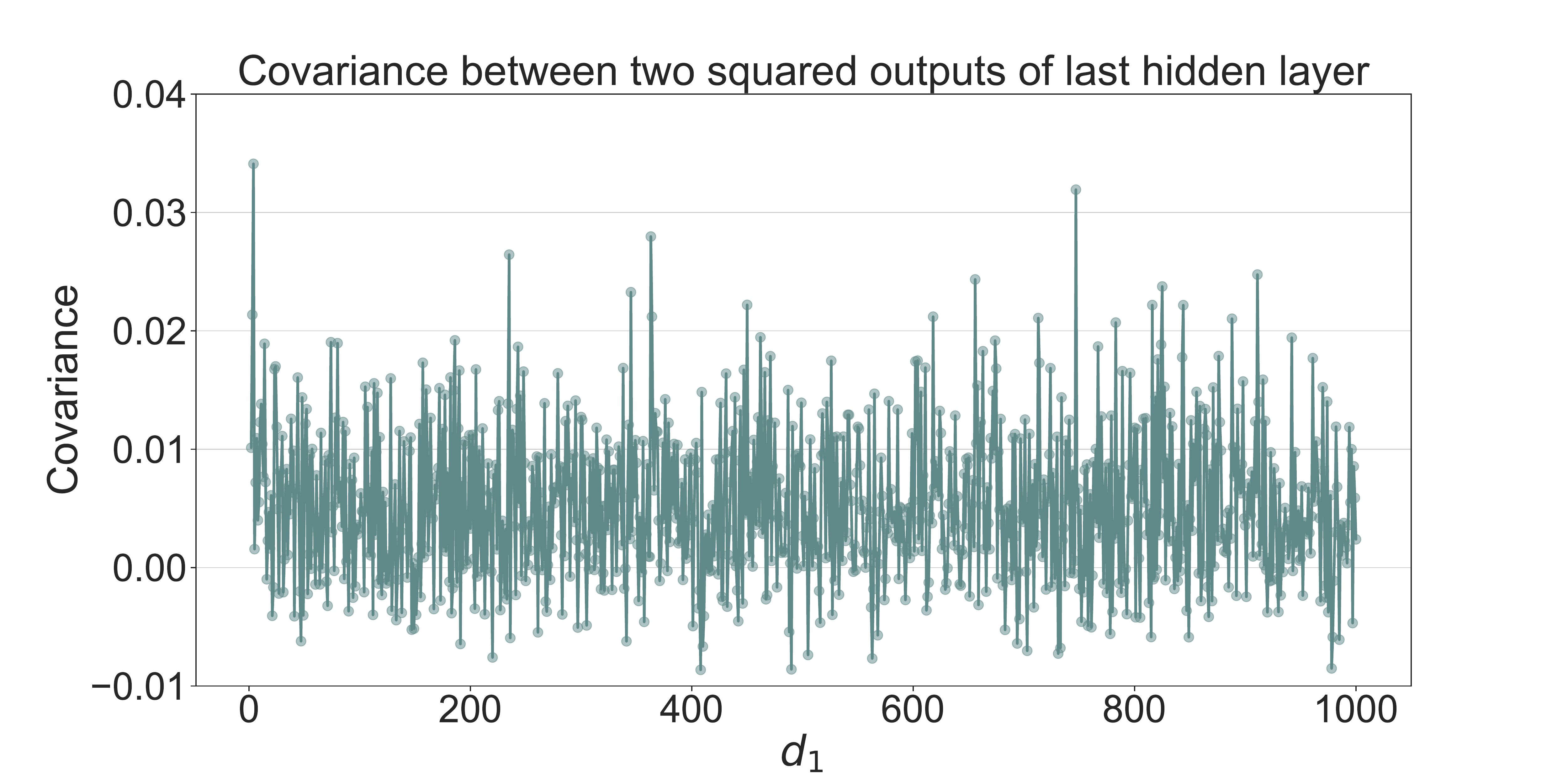}
    \caption{Covariance between two squared outputs of last hidden layer. Only the first hidden layer width grows.}
    \label{fig:covariance-1}
\end{figure}

\begin{figure}[H]
    \centering
    \includegraphics[width=0.8\textwidth]{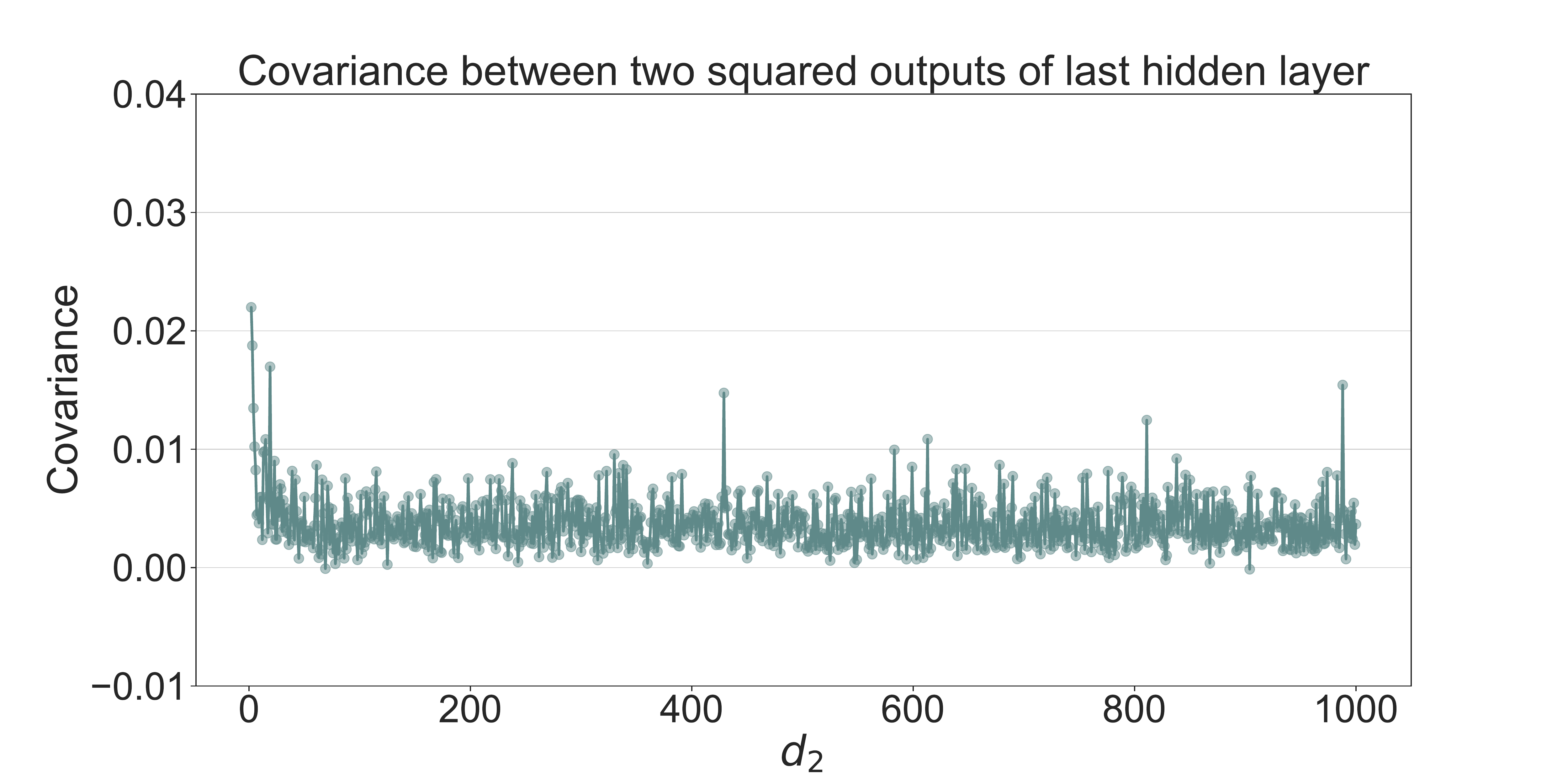}
    \caption{Covariance between two squared outputs of last hidden layer. Only the second hidden layer width grows.}
    \label{fig:covariance-2}
\end{figure}

\begin{figure}[H]
    \centering
    \includegraphics[width=0.8\textwidth]{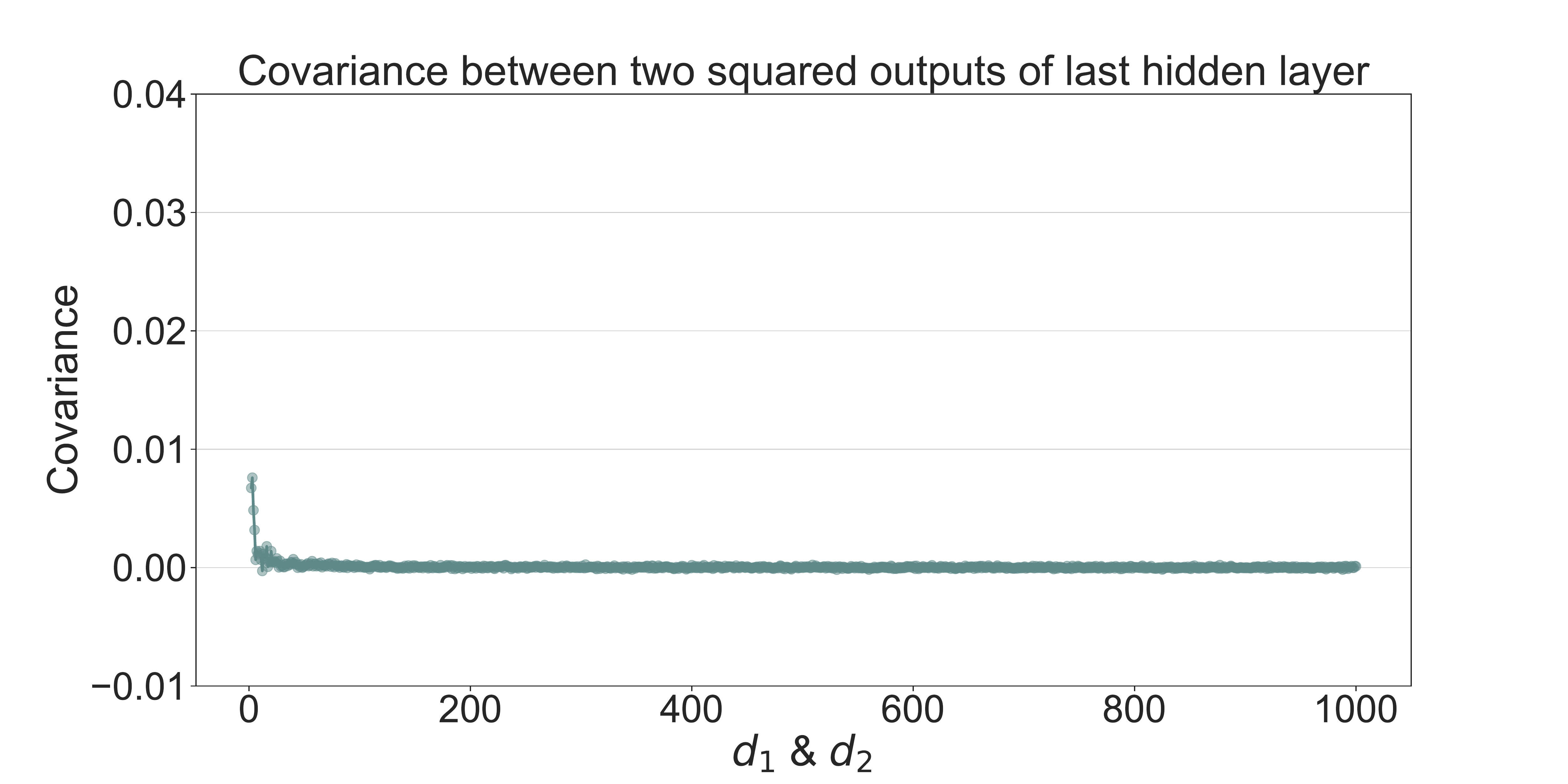}
    \caption{Covariance between two squared outputs of last hidden layer. Both the first and second hidden layer width grow.}
    \label{fig:covariance-12}
\end{figure}

\bibliographystyle{alpha}
\bibliography{ref}

\newpage

\newpage

\section{Supplemental Material}

\subsection{ Measurable space of distribution functions}

\begin{remark}\label{fundamental}
The following are fundamental facts of probability theory.
\begin{itemize}
    \item For every probability measure $\mu$ on $\bR$, we can construct a one-dimensional distribution function $F_\mu$ by $F_\mu(t) = \mu((-\infty, t])$. On the other hand, given a one-dimensional distribution function $F$, we can construct a probability measure $\mu_F$ on $\bR$ satisfying $F_{\mu_{F}} = F$. 
    Since we also have $\mu_{F_\mu} = \mu$, this correspondence is a bijection. 
     Note that we can also construct a $\bR$-valued random variable whose associated measure is $\mu_F$, which we denote by $X_F$.  
    \item Given a measurable space $\Omega$, we can equip the countably infinite product $\Omega^{\bN}$ with the smallest family of measurable sets such that each projection $\Omega^{\bN} \longrightarrow \Omega$ is measurable. We denote such a measurable space by $\Omega^{\bN}$. In particular, $\bR^{\bN}$ is a measurable space.
    \item Given a random variable $X : \Omega \longrightarrow \bR$, we can construct a probability measure on $\Omega^{\bN}$ such that random variables $X_1, X_2, \ldots : \Omega^{\bN} \longrightarrow \Omega \longrightarrow \bR$ are i.i.d. and their distributions are same as $X$'s. The existence of a sequence in Definition \ref{dfdist} (4) is guaranteed by this fact. 
    \end{itemize}
\end{remark}

\begin{definition}
Let $\cM_{1}(\bR)$ be the set of probability measures on the Borel space $(\bR, \cB(\bR))$. We endow
$\cM_{1}(\bR)$ the $\sigma$-field generated by all projection maps $\pi_A ; \mu \mapsto \mu(A)$, $A \in \cB(\bR)$. 
\end{definition}
\begin{lemma}\label{generating}
Every projection $\pi_A \colon \cM_{1}(\bR) \longrightarrow \bR$ is measurable if and only if the projection $\pi_{(-\infty, t]}$ is measurable for all $t \in \bR$.
\end{lemma}
\begin{proof}
Suppose that $\pi_{(-\infty, t]}$ is measurable for all $t \in \bR$, and we will show that  $\pi_A$ is measurable for any $A \in \cB(\bR)$. Note that $\pi_{\bR \setminus A} = 1 - \pi_A$ is measurable when $\pi_A$ is measurable. Hence it suffices to show that $\pi_{\cap_iA_i}$ and $\pi_{\cup_i A_i}$ are measurable when $A_i$'s are of the form $(-\infty, t_i]$ or $(t_i, \infty)$ for $i \in \bN$, since $\cB(\bR)$ is generated by $\{(-\infty, t]\}_{t\in \bR}$. It is obvious that $\pi_{\cap_i^kA_i}$ and $\pi_{\cup_i^k A_i}$ are measurable for any $k \in \bN$. Since the sequences $\{\pi_{\cap_i^kA_i}\}_{k \in \bN}$ and $\{\pi_{\cup_i^k A_i}\}_{k \in \bN}$ converge point-wise to $\pi_{\cap_iA_i}$ and $\pi_{\cup_i A_i}$ respectively, the following Lemmma \ref{ptwconverge} shows that they are measurable. 
\end{proof}
\begin{lemma}\label{ptwconverge}
Let $E$ be any measurable space, and let $f_i : E \longrightarrow (\overline{\bR}, \cB(\overline{\bR}))$ be measurable maps for $i \in \bN$. If the sequence $\{f_i\}_{i \in \bN}$ converge point-wise to a map $f$, then $f$ is measurable.
\end{lemma}
\begin{proof}
Note that $\sup_i f_i$ is measurable, hence so is $\limsup_i f_i$. Thus we obtain that $f = \limsup_i f_i$ is measurable.
\end{proof}
\begin{proposition}\label{isom}
The bijection in Remark \ref{fundamental}
\[
\varphi : \cM_{1}(\bR) \stackrel{\sim}{\longrightarrow} \fF; \mu \mapsto \mu((-\infty,t])
\]
is an isomorphism of measure spaces. 
\end{proposition}
\begin{proof}
The following shows that $\varphi$ is measurable: 
\begin{align*}
\varphi^{-1}(\fF(x, y)) 
&= \{\mu \in \cM_1(\bR) \mid \mu((-\infty, x ]) \leq y\} 
\\
&= \pi_{(-\infty, x]}^{-1}((-\infty, y]).
\end{align*}
To show that $\varphi^{-1}$ is measurable, by Lemma \ref{generating}, it is enough to show that $\varphi(\pi_{(-\infty, x]}^{-1}((-\infty, y]))$ is measurable in $\fF$. 
Since we have $\pi_{(-\infty, x]}^{-1}((-\infty, y]) = \{\mu \in \cM_1(\bR) \mid \mu((-\infty, x ]) \leq y\}$, we obtain that $\varphi(\pi_{(-\infty, x]}^{-1}((-\infty, y])) \subset \fF(x, y)$. The inverse inclusion follows from the definition of the inverse map of $\varphi$.
\end{proof}
\begin{proposition}\label{intmeasurable}
For any measurable map $f : \bR \longrightarrow \bR$, the map $s_f : \cM_1(\bR) \longrightarrow \overline{\bR}\  ; \mu \mapsto \int f d\mu$ is measurable.
\end{proposition}
\begin{proof}
If we put $f^+(x) := \max\{f(x), 0\}$ and $f^-(x) := \max\{-f(x), 0\}$, then we have $f = f^+ - f^-$. Hence we may assume that $f \geq 0$. 
Let $\{f_n\}_n$ be a simple function approximation of $f$, that is a sequence of simple functions converging point-wise to $f$ with $f_n \leq f$. By the monotone convergent theorem, we have 
\begin{align*}
    s_f(\mu) &= \int f d\mu \\
    &= \int \lim_n f_n d\mu \\
    &= \lim_n \int f_n d\mu.
\end{align*}
Since $\int f_n d\mu$ is a summation of $\pi_A$'s for some Borel sets $A$'s, it is measurable. Hence $s_f$ is also measurable by Lemma \ref{ptwconverge}.
\end{proof}

\begin{proposition}\label{expmeasurable}
For any measurable map $f : \bR \longrightarrow \bR$ and any $r \in \overline{\bR}$, the set $\{F \in \fF \mid \bE[f(X_F)] \leq r\}$ is measurable in $\fF$.
\end{proposition}
\begin{proof}
Note that we have  $\bE[f(X_F)] = \int_\bR f(x) d\mu_F (x)$, which shows that the map $\fF \longrightarrow \overline{\bR} ; F \mapsto \bE[f(X_F)]$ is identified with $s_f$ by Proposition \ref{isom}. Since the map $s_f$ is measurable by Proposition \ref{intmeasurable}, the set $\{F \in \fF \mid \bE[f(X_F)] \leq r\} = s_f^{-1}((-\infty, r])$ is measurable in $\fF$.
\end{proof}

\subsection{$L^p$-spaces and their intersection}
\begin{definition}
Let $\Omega$ be a probability space. For $p \in \bN$, we put 
\[
L^p(\Omega) := \{X : \Omega \longrightarrow \bR \mid \bE[|X|^p] < \infty\}. 
\]
Then we have $L^{p+1}(\Omega) \subset L^p(\Omega)$, and we put $\tilde{L}(\Omega) := \cap_{p\in \bN}L^p(\Omega)$.
\end{definition}

\begin{remark}\label{vector}
By an elementary inequality $|s + t|^p \leq 2^{p-1}(|s|^p + |t|^p)$ for $s, t \in \bR$, we have that $\tilde{L}(\Omega)$ is a real vector space.
\end{remark}

\begin{lemma}\label{bddpoly}
If a function $\sigma : \bR \longrightarrow \bR$ satisfies $|\sigma| \leq |u|$ for some polynomial $u$, then $\sigma$ induces a map $\tilde{L}(\Omega) \longrightarrow \tilde{L}(\Omega); X \mapsto \sigma \circ X$.
\end{lemma}
\begin{proof}
We show that $\bE[|\sigma(X)|^p] \leq \bE[|u(X)|^p] < \infty$ for any $p \in \bN$. Since $\tilde{L}(\Omega)$ is a vector space by Remark \ref{vector}, it suffices to show that $X^n \in \tilde{L}(\Omega)$ for any  $X \in \tilde{L}(\Omega)$ and $n \in \bN$, which follows from the definition of $\tilde{L}(\Omega)$.
\end{proof}

\begin{corollary}\label{prodl}
$\tilde{L}(\Omega)$ is an algebra over $\bR$.
\end{corollary}
\begin{proof}
For $X, Y \in \tilde{L}(\Omega)$, we have $XY = \frac{1}{2}((X + Y)^2 - X^2 - Y^2) \in \tilde{L}(\Omega)$ by Lemma \ref{bddpoly} and Remark \ref{vector}. This completes the proof.
\end{proof}
\subsection{Maximum Mean Discrepancy}\label{sub1}
Let $\mathcal{M}$ be the set of all probability measures on $\bR$ which are absolutely continuous. To define a distance on $\mathcal{M}$, we construct a $\mathbb{C}$-Hilbert space $\mathcal{H}$ and a map $m : \mathcal{M} \longrightarrow \mathcal{H} ; p \mapsto m_p$ satisfying the following.
\begin{enumerate}
    \item $\mathcal{H} \subset \mathrm{Map}(\bR, \bC)$ is a {\it reproducing kernel Hilbert space (RKHS)}.
    \item for any $f \in \mathcal{H}$, we have $\langle f, m_p \rangle_{\mathcal{H}} = \bE_p[f]$.
    \item $m$ is injective.
\end{enumerate}
Then we can introduce a distance function $d$ on $\mathcal{M}$ by 
\[
d(p, q) := \|m_p - m_q\|_{\mathcal{H}},
\]
which is called {\it maximum mean discrepancy} in \cite{Gretton12}. By the following lemma, we can compute this metric function from the kernel $k$ of the RKHS $\mathcal{H}$.

\begin{proposition}\label{formula}
\begin{equation}\label{dist}
d(p, q) = \bE_{x, x'\sim p}[k(x, x')] - 2\bE_{x \sim p, y\sim q}[\mathrm{Re}\,  k(x, y)] + \bE_{y, y'\sim q}[k(y, y')],
\end{equation}
where $\mathrm{Re}$ denotes the real part of a function.
\begin{proof}
For any $x \in \bR$, let $\varphi_x \in \mathcal{H}$ be a function satisfying $f(x) = \langle f, \varphi_x \rangle_{\mathcal{H}}$ and $\varphi_x(y) = k(x, y)$ for any $f \in \mathcal{H}$. Such $\varphi_x$'s exist since $\mathcal{H}$ is the RKHS associated with $k$. Then we have
\begin{align*}
    d(p, q)  &= \|m_p - m_q\|_{\mathcal{H}} \\
    &= \langle m_p - m_q, m_p - m_q \rangle_{\mathcal{H}} \\
    &= \langle m_p , m_p  \rangle_{\mathcal{H}} - 2\mathrm{ Re}\ \langle m_p , m_q  \rangle_{\mathcal{H}} + \langle m_q , m_q  \rangle_{\mathcal{H}} \\
    &= \bE_{x\sim p}[m_p(x)] - 2\bE_{x \sim p}[\mathrm{ Re}\  m_q(x)] + \bE_{y \sim q}[m_q(y)] \\
    &= \bE_{x\sim p}[\langle m_p, \varphi_x\rangle_{\mathcal{H}}] - 2\bE_{x \sim p}[\mathrm{ Re}\  \langle m_q, \varphi_x\rangle_{\mathcal{H}}] + \bE_{y \sim q}[\langle m_q, \varphi_y \rangle_{\mathcal{H}}] \\
    &= \bE_{x, x'\sim p}[k(x, x')] - 2\bE_{x \sim p, y\sim q}[\mathrm{ Re}\  k(x, y)] + \bE_{y, y'\sim q}[k(y, y')],
\end{align*}
here we applied $\langle f, m_p \rangle_{\mathcal{H}} = \bE_p[f]$ for any $f \in \mathcal{H}$ to the fourth and sixth line, and $f(x) = \langle f, \varphi_x \rangle_{\mathcal{H}}$ $f \in \mathcal{H}$ to the fifth line. This completes the proof.
\end{proof}
\end{proposition}\label{expbdd}
In the following, we explain the construction of $\mathcal{H}$ and $m$ satisfying (1)--(3) above. Let $k(x, y) = \exp(-(x-y)^2/2)$, and let $\mathcal{H}$ be the $\bC$-RKHS associated with $k$. Then we have the following.
\begin{proposition}
For any $p \in \mathcal{M}$, there exists $m_p \in \mathcal{H}$ such that $\langle f, m_p \rangle_{\mathcal{H}} = \bE_p[f]$ for any $f \in \mathcal{H}$.
\end{proposition}
\begin{proof}
We show that the map $\mathcal{H} \longrightarrow \bC ; f \mapsto \bE_p[f]$ is bounded. Then Riez's lemma implies the statement. Let $f \in \mathcal{H}$ with $\|f\|_{\mathcal{H}} = 1$. Then we have 
\begin{align*}
    |\bE_p[f]| &\leq \bE_p[|f|] = \int_{\bR}|f(x)| dp(x) \\
    &= \int_{\bR}|\langle f, \varphi_x \rangle_{\mathcal{H}}| dp(x) \\
    &\leq  \int_{\bR} \sqrt{\langle f, f \rangle_{\mathcal{H}}}\sqrt{\langle \varphi_x, \varphi_x \rangle_{\mathcal{H}}} dp(x) \\
    &= \int_{\bR} \sqrt{k(x,x)}dp(x) \\
    &\leq  \int_{\bR}dp(x) = 1 < \infty,
\end{align*}
here we applied Cauchy--Schwarz inequality to the third line. This completes the proof.
\end{proof}
By Proposition \ref{expbdd}, we can define $m : \mathcal{M} \longrightarrow \mathcal{H}$ by $m(p) = m_p$. Finally, we show the injectivity of $m$. It reduces to show that $\mathcal{H}\cap L^2(\bR, p)$ is dense in $L^2(\bR, p)$ for any $p \in \mathcal{M}$ by the following lemma.
\begin{lemma}\label{denseinj}
If $\mathcal{H}\cap L^2(\bR, p)$ is dense in $L^2(\bR, p)$ with respect to the $L^2(\bR, p)$-norm for any $p \in \mathcal{M}$, then the map $m$ is injective. That is, $m_p = m_q$ implies that $p = q$ as measures.
\end{lemma}
\begin{proof}
Let $A \subset \bR$ be a Borel set. Then $1_A \in L^2(\bR, p)$ for any $p \in \mathcal{M}$. Here $1_A(x) = \begin{cases}1 & x \in A \\ 0 & x \not\in A\end{cases}$. Let $p, q \in \mathcal{M}$. For any $\varepsilon > 0$, there exists $f \in \mathcal{H}\cap L^2(\bR, p)$ such that 
\[
\begin{cases}
\|f - 1_A\|_{L^2(\bR, p)} < \varepsilon, \\
\|f - 1_A\|_{L^2(\bR, q)} < \varepsilon, 
\end{cases}
\]
by the density assumption. With the elementary fact that $\|\cdot\|_{L^1(\bR, p)} \leq \|\cdot\|_{L^2(\bR, p)}$, we obtain
\begin{align*}
    |\bE_p[f] - p(A)| &= |\int f - 1_A dp | \\
    &\leq  \int |f - 1_A | dp < \varepsilon,
\end{align*}
and similarly for $q$. Then $m_p = m_q$ implies 
\[
|p(A) - q(A)| \leq |\bE_p[f] - p(A)| + |\bE_q[f] - q(A)| \leq 2\varepsilon,
\]
which implies that $p = q$. This completes the proof.
\end{proof}
To check the density condition, we take a particular subset of $\mathcal{H}\cap L^2(\bR, p)$.
\begin{lemma}\label{edense}
For any $p \in \mathcal{M}$, the linear span of the set $\{e^{ixt}\}_{t\in \bR}$ is dense in $L^2(\bR, p)$.
\end{lemma}
\begin{proof}
Let $S$ be the linear span of $\{e^{ixt}\}_{t\in \bR}$. It reduces to show that $S^{\perp} = 0$ since it implies that $L^2(\bR, p) = \overline{S} \oplus \overline{S^{\perp}} = \overline{S}$. Note that  $f \in S^{\perp}$ is equivalent to that 
\begin{equation}\label{fourier}
    \int_{\bR} e^{-ixt}f(x)p(x) = \int_{\bR} e^{-ixt}f(x)\rho(x)dx = 0,\   \forall t \in \bR,
\end{equation}
here $\rho$ is a density function of $p$. Since $f \in L^2(\bR, p) \subset L^1(\bR, p)$, we have $f\rho \in L^1(\bR, dx)$. Hence (\ref{fourier}) is equivalent to $\mathcal{F}f\rho = 0$, where $\mathcal{F}$ denotes the Fourier transform on $L^1(\bR, dx)$, which implies that $f\rho = 0$ almost everywhere with respect to the Lebesgue measure \cite{Igari98}. Therefore we obtain 
\[
\|f\|_{L^2(\bR, p)} = \int_\bR |f(x)|^2\rho(x)dx = \int_\bR \overline{f(x)}f(x)\rho(x)dx = 0,
\]
which implies that $f = 0$ in $L^2(\bR, p)$. This completes the proof.
\end{proof}
\begin{lemma}\label{erhodense}
For any $p \in \mathcal{M}$, the linear span of the set $\{e^{ixt} \exp(-x^2/2\tau^2)\}_{t\in \bR, \tau > 0}$ is dense in $L^2(\bR, p)$.
\end{lemma}
\begin{proof}
We show that $\bE_p|e^{ixt} - e^{ixt} \exp(-x^2/2\tau^2)|^2 \to 0$ as $\tau \to \infty$ for any $p \in \mathcal{M}$. Then Lemma \ref{edense} implies the statement. Since we have $|e^{ixt} - e^{ixt} \exp(-x^2/2\tau^2)|^2 \leq (1 + 1)^2 = 4$, Lebesgue's convergence theorem implies $\lim_{\tau \to \infty }\bE_p|e^{ixt} - e^{ixt} \exp(-x^2/2\tau^2)|^2 = 0$. This completes the proof.
\end{proof}

We show that $\{e^{ixt} \exp(-x^2/2\tau^2)\}_{t\in \bR, \tau > 0} \subset \mathcal{H}$ for large $\tau$ by explicitly representing $\mathcal{H}$ as follows. Note that $k(x, y) = \exp(-(x-y)^2/2)$ is the characteristic function of the normal distribution. Namely we have 
\begin{align*}
    k(x, y) &= \exp\left(-\frac{(x-y)^2}{2}\right) \\
    &= \int_\bR e^{i(x-y)\xi}\frac{1}{\sqrt{2\pi}}\exp\left(-\frac{\xi^2}{2}\right) d\xi \\
    &= \int_\bR e^{ix\xi}\overline{e^{iy\xi}} d\varphi(\xi) \\
    &= \langle e^{ix\xi}, e^{iy\xi}\rangle_{L^2(\bR, d\varphi)},
\end{align*}
where $\varphi(\xi) = \exp(-\xi^2/2)/\sqrt{2\pi}$. Now we consider a map $j : L^2(\bR, d\varphi) \longrightarrow \mathrm{ Map}(\bR, \bC)$ defined by 
\[
(jF)(x) = \langle F(\xi), e^{ix\xi}\rangle_{L^2(\bR, d\varphi)},
\]
which is injective by Lemma \ref{edense}. Then $\mathrm{ Im}j$ equipped with the inner product 
\[
\langle jF, jG \rangle_{\mathrm{ Im}j} = \langle F, G\rangle_{L^2(\bR, d\varphi)}
\]
is the RKHS with kernel $k$. It is checked as follows. For any $jF \in \mathrm{ Im}j$, we have 
\[
jF(x) = \langle F(\xi), e^{ix\xi}\rangle_{L^2(\bR, d\varphi)} = \langle jF, je^{ix\xi} \rangle_{\mathrm{ Im}j}.
\]
We also have $\langle je^{ix\xi}, je^{iy\xi} \rangle_{\mathrm{ Im}j} = k(x, y)$. Hence we can identify $\mathcal{H}$ with $\mathrm{ Im}j$.
\begin{lemma}\label{erhoinh}
For any $t \in \bR$ and $\tau > 1$, we have $e^{ixt} \exp(-x^2/2\tau^2) \in \mathrm{ Im}j$.
\end{lemma}
\begin{proof}
Note that $e^{ixt} \exp(-x^2/2\tau^2)$ is the characteristic function of some normal distribution. Namely we have 
\begin{align*}
    e^{ixt} \exp(-x^2/2\tau^2) &= \int_\bR e^{ix\xi}\frac{1}{\sqrt{2\pi \frac{1}{\tau^2}}}\exp\left(-\frac{\tau^2(\xi - t)^2}{2}\right) d\xi \\
    &= \int_\bR -\tau e^{-ix\xi}\frac{1}{\sqrt{2\pi}}\exp\left(-\frac{\tau^2(\xi + t)^2}{2}\right) d\xi \\
    &= \int_\bR -\tau e^{-ix\xi}\exp\left(-\frac{\tau^2(\xi + t)^2}{2} + \frac{\xi^2}{2}\right) d\varphi(\xi) \\
    &= \langle -\tau \exp\left(-\frac{\tau^2(\xi + t)^2}{2} + \frac{\xi^2}{2}\right) , e^{ix\xi}  \rangle_{L^2(\bR, d\varphi)} ,
\end{align*}
since we have $-\tau \exp(-\tau^2(\xi + t)^2/2 + \xi^2/2) \in L^2(\bR, d\varphi)$ for $\tau >1$. Hence $e^{ixt} \exp(-x^2/2\tau^2) \in \mathrm{Im} j$. This completes the proof.
\end{proof}
Now we apply the method to the distribution $p, q$ in the following.
\begin{itemize}
    \item $q$ is a distribution on the output layer of the neural network with 3 hidden layers with $d_1, d_2, d_3$ components respectively. We set the output layer $1$-dimensional for simplicity.
    \item $p$ is a normal distribution $N(0, \sigma^2)$ for some $\sigma > 0$.
\end{itemize}
Now we take samples $\mathcal{Y} = \{y_n\}_n$ from $q$, and take $\sigma^2$ as the variance of the sample $\mathcal{Y}$. By the law of large numbers, the above quantity is approximated by the following for large $n$ :

\begin{equation}\label{dist2}
 \bE_{x, x'\sim p}\left[\exp\left(-\frac{(x-x')^2}{2}\right)\right] - \frac{2}{n}\sum_{i=1}^{n}\bE_{x \sim p}\left[\exp\left(-\frac{(x-y_i)^2}{2}\right)\right] + \frac{1}{n^2}\sum_{i, j = 1}^{n}\exp\left(-\frac{(y_i-y_j)^2}{2}\right).
\end{equation}
\begin{lemma}\label{calc1}
\[
\bE_{x \sim N(0, \sigma^2)}\left[\exp\left(-\frac{(x-y)^2}{2}\right)\right] = \frac{1}{\sqrt{\sigma^2 + 1}} \exp\left(-\frac{y^2}{2(\sigma^2 + 1)}\right).
\]

\end{lemma}
\begin{proof}
We have
\begin{align*}
    \bE_{x \sim N(0, \sigma^2)}\left[\exp\left(-\frac{(x-y)^2}{2}\right)\right] &= \frac{1}{\sqrt{2\pi\sigma^2}}\int_{\bR} \exp\left(-\frac{x^2}{2\sigma^2}\right) \exp\left(-\frac{(x-y)^2}{2}\right)dx \\
    &= \frac{1}{\sqrt{2\pi\sigma^2}}\int_{\bR} \exp\left(-\frac{x^2}{2\sigma^2}\right) \int_{\bR}\frac{1}{\sqrt{2\pi}}e^{i(x-y)\xi}\exp\left(-\frac{\xi^2}{2}\right)d\xi dx \\
    &= \frac{1}{\sqrt{2\pi}}\int_\bR  \exp\left(-\frac{\xi^2\sigma^2}{2}\right)e^{-iy\xi}\exp\left(-\frac{\xi^2}{2}\right)d\xi \\
    &= \frac{1}{\sqrt{2\pi}}\int_\bR \exp\left(-\frac{\xi^2(\sigma^2 + 1)}{2}\right)e^{-iy\xi}d\xi \\
    &= \frac{1}{\sqrt{\sigma^2 + 1}} \exp\left(-\frac{y^2}{2(\sigma^2 + 1)}\right),
\end{align*}
where we applied the characteristic function formula of normal distribution to the first and fourth line, and Fubini-Tonelli theorem to the second line. This completes the proof.
\end{proof}
\begin{lemma}\label{calc2}
\[
\bE_{x, x' \sim N(0, \sigma^2)}\left[\exp\left(-\frac{(x-x')^2}{2}\right)\right] = \frac{1}{\sqrt{2\sigma^2 + 1}}.
\]

\end{lemma}
\begin{proof}
By Lemma \ref{calc1}, we have 
\begin{align*}
    \bE_{x, x' \sim N(0, \sigma^2)}\left[\exp\left(-\frac{(x-x')^2}{2}\right)\right] &= \bE_{ x' \sim N(0, \sigma^2)}\left[\frac{1}{\sqrt{\sigma^2 + 1}} \exp\left(-\frac{x'^2}{2(\sigma^2 + 1)}\right)\right] \\
    &= \int_\bR \frac{1}{\sqrt{\sigma^2 + 1}} \exp\left(-\frac{x'^2}{2(\sigma^2 + 1)}\right) \frac{1}{\sqrt{2\pi\sigma^2}}\exp\left(-\frac{x'^2}{2\sigma^2}\right)dx' \\
    &= \frac{1}{\sqrt{\sigma^2(\sigma^2 + 1)}}\frac{1}{\sqrt{2\pi}} \int_\bR \exp\left(-\frac{x'^2}{2}\left(\frac{1}{\sigma^2} + \frac{1}{\sigma^2 + 1}\right)\right) dx' \\
    &= \frac{1}{\sqrt{2\sigma^2 + 1}}.
\end{align*}
This completes the proof.
\end{proof}

Now we obtain the following proposition.  

\begin{proposition}\label{quantity}
The quantity (\ref{dist2}) is equal to the following:
\begin{equation}\label{dist3}
    \frac{1}{\sqrt{2\sigma^2 + 1}} - \frac{2}{n}\sum_{i=1}^{n}\frac{1}{\sqrt{\sigma^2 + 1}} \exp\left(\frac{y_i^2}{2(\sigma^2 + 1)}\right) + \frac{1}{n^2}\sum_{i, j = 1}^{n}\exp\left(-\frac{(y_i-y_j)^2}{2}\right),
\end{equation}
where $\sigma^2$ is the variance of the samples $\mathcal{Y}$.
\end{proposition}
\end{document}